\documentclass[10pt]{article} 

\usepackage[accepted]{rlj} 

\usepackage{amssymb}            
\usepackage{mathtools}          
\usepackage{mathrsfs}           
\usepackage{graphicx}           
\usepackage{subcaption}         
\usepackage[space]{grffile}     
\usepackage{url}                
\usepackage{lipsum}             

\usepackage{shortcuts}
\usepackage[colorinlistoftodos,textsize=tiny]{todonotes}

\title{Rectifying Regression in Reinforcement Learning}

\setrunningtitle{Rectifying Regression in Reinforcement Learning}

\author{Alex Ayoub\textsuperscript{1,2$\dagger$}, David Szepesvári\textsuperscript{1,2$\dagger$}, Alireza Bakhtiari\textsuperscript{1}, Csaba Szepesvári\textsuperscript{1,2}, Dale Schuurmans\textsuperscript{1,2}}

\emails{\{aayoub,dszepesv\}@ualberta.ca}

\affiliations{
$^{1}$\textbf{Department of Computing Science, University of Alberta, Canada}\\
$^{2}$\textbf{Alberta Machine Intelligence Institute (Amii)}
\par 
$^\dagger$ Equal Contribution}

\contribution{
    We demonstrate certain cross entropy losses can accelerate convergence under certain structural assumptions, supported by negative results for the purely mean-focused squared loss.
    }
    {
    We build upon a recent line of theoretical \citep{foster2021efficient,ayoub2024switching,wang2024more} and empirical \citep{bellemare2017distributional,dabney2018distributional,farebrother2024stop} research showing that value learning with certain loss functions can yield faster convergence rates under specific structural assumptions, such as the optimal policy achieving the maximum possible value or having low variance returns. We complement these findings by providing lower bounds that link these convergence rates to the chosen regression objective---in this case mean absolute error and mean squared error.
    }

\contribution{
    We provide empirical results showing that value-based methods using log-loss (and its reparameterized multi-class variant) can outperform squared-loss methods in a linear batch reinforcement learning setting (inverted pendulum with Fourier features).
    }
    {
    The work of \cite{lyle2019comparative} suggest that, in linear reinforcement learning, cross-entropy losses (e.g., binary or categorical) perform on par with squared loss and that their advantages appear primarily in deep reinforcement learning settings. However, our theoretical and empirical findings suggest a more subtle situation: in linear reinforcement learning, cross-entropy losses can outperform the canonical squared loss. Our experiments are limited to a single environment (inverted pendulum) with Fourier features.
    }


\summary{This paper investigates the impact of the loss function in value-based methods for reinforcement learning through an analysis of underlying prediction objectives. We theoretically show that mean absolute error is a better prediction objective than the traditional mean squared error for controlling the learned policy's suboptimality gap. Furthermore, we present results that different loss functions are better aligned with these different regression objectives: binary and categorical cross-entropy losses with the mean absolute error and squared loss with the mean squared error.
We then provide empirical evidence that algorithms minimizing these cross-entropy losses can outperform those based on the squared loss in linear reinforcement learning.
}

\begin{document}

\makeCover  
\maketitle  


\begin{abstract}
This paper investigates the impact of the loss function in value-based methods for reinforcement learning through an analysis of underlying prediction objectives. We theoretically show that mean absolute error is a better prediction objective than the traditional mean squared error for controlling the learned policy's suboptimality gap. Furthermore, we present results that different loss functions are better aligned with these different regression objectives: binary and categorical cross-entropy losses with the mean absolute error and squared loss with the mean squared error.
We then provide empirical evidence that algorithms minimizing these cross-entropy losses can outperform those based on the squared loss in linear reinforcement learning.
\end{abstract}

\section{Introduction}
Value-based methods are ubiquitous in reinforcement learning (RL) \citep{sutton2018reinforcement} and contextual bandits \citep{lattimore2020bandit}, where the goal is to predict rewards and then choose actions that maximize expected returns. The ``natural objective function'' (Chapter 9.2 of \citep{sutton2018reinforcement}) in these settings is the mean squared error (MSE). Traditionally in RL, value-based algorithms that leverage function approximation aim to minimize the MSE between a learned value function and observed returns, as popularized by algorithms such as $Q$-learning and its extension to function approximation \citep{ernst2005tree,riedmiller2005neural,Mnih2015HumanlevelCT}. Despite its success in many practical scenarios, recent theoretical \citep{foster2021efficient,ayoub2024switching} and empirical \citep{farebrother2024stop} findings highlight that minimizing \emph{MSE} on the data can yield worse
decision-making performance (in terms of expected returns) than alternative empirical losses. In particular, minimizing alternative losses can achieve faster learning when the optimal value $v^\star$ is close to the maximum possible value.

A promising alternative is the \emph{log-loss} (cross-entropy loss), which we argue aligns more closely with the mean absolute error (MAE)---a tighter surrogate for decision quality. Indeed, it has been shown that, under certain problem structures, controlling the MAE can lead to faster convergence \citep{foster2021efficient} and a smaller suboptimality gap \citep{ayoub2024switching} than controlling MSE. Building on these insights, we analyze \emph{value-based methods} that employ log-loss and demonstrate their advantages over those using squared loss. Our exposition initially restricts attention to the offline contextual bandit (or reward-sensitive classification \citep{elkan2001foundations}) setting, where one aims to learn a policy from a fixed dataset of context–reward pairs. However, the core ideas and proofs naturally extend to more general RL problems with function approximation \citep{antos2007fitted,chen2019information,ayoub2024switching}.

We propose a \emph{reparameterized} version of the categorical cross-entropy loss (\(\ellcat\)), which learns a multi-category distribution over possible outcomes while still accurately recovering the mean of a bounded random variable as was similarly done by \cite{lyle2019comparative}. This addresses the irreducible bias found in purely bin-based approaches.

In summary, our work contributes new insights into why certain losses such as log-loss are better suited than squared loss in certain decision-making settings, highlights methods to reparameterize cross entropy losses that can be utilized simultaneously for classification and regression, and highlights broader implications for distributional RL. Although we focus on contextual bandits for clarity, the same techniques extend naturally to RL, offering an explanation as to why certain loss functions can improve value-based algorithms. 

\section{Definitions}
In this section, we formally define regression and classification problems. We then introduce a solution to the classification problem that leverages a regression oracle, highlighting the connection between these two paradigms. By distinguishing problems from their solutions, we establish a formal framework for understanding how an appropriately chosen regression objective can guide the development of more effective decision-making algorithms.

\subsection{Problem: Regression}
Consider a supervised learning problem where the set of contexts is denoted by $\mathcal{X}$. The learner is provided with a dataset $D_n = ((X_1,Y_1),\dotsc,(X_n,Y_n))$, consisting of $n$ independently and identically distributed (i.i.d.) context-label pairs, sampled from a distribution $\mathcal{P}^{\otimes n}$. Each sample $(X_i, Y_i)$ satisfies $X_i \in \mathcal{X}$ and $Y_i \in [0,1]$, where $\mathcal{P}$ is an element of $\mathcal{M}_1(\mathcal{X} \times [0,1])$, the space of probability measures over $\mathcal{X} \times [0,1]$. We assume $\mathcal{M}_1(\cdot)$ is defined over an appropriately equipped $\sigma$-algebra.

Define the conditional expectation of the label given the context as
\[
f^\star(x) = \mathbb{E}[Y_1 | X_1 = x], \quad \forall x \in \mathcal{X}.
\]
A score function $f$ maps contexts to predicted labels in $[0,1]$, formally defined as $f: \mathcal{X} \to [0,1]$. 

Throughout this paper, we assume that we are given a \textit{realizable class} (\cref{ass:realizability}) of score functions $\mathcal{F} \subseteq [0,1]^{\mathcal{X}}$. Let $P_X \in \mathcal{M}_1(\mathcal{X})$ denote the marginal distribution over contexts.
\begin{assumption}[Realizability]\label{ass:realizability}
    $f^\star \in \mcF$.
\end{assumption}
For $p \geq 1$, the prediction error (or value error) of a function $f \in \mathcal{F}$ is defined as
\[
\mathsf{VE}_p(f) = \mathbb{E} \bigl|f(X) - f^\star(X)\bigr|^p 
= \int |f(x) - f^\star(x)|^p \, P_X(dx).
\]
Thus, $\mathsf{VE}_p(f)$ quantifies the deviation of $f$ from the optimal predictor $f^\star$. The objective in regression is to learn a function $\hat{f}$ that minimizes the prediction error $\mathsf{VE}_p(\hat{f})$. Special cases of this error metric include the mean squared error  when $p=2$
\[
\mathsf{VE}_2(f) = \mathbb{E}[(f(X) - f^\star(X))^2],
\]
and the mean absolute error when $p=1$
\[
\mathsf{VE}_1(f) = \mathbb{E} \bigl|f(X) - f^\star(X)\bigr|.
\]

\subsection{Problem: Classification}

Consider a (reward-sensitive) classification problem, where the set of contexts is denoted by $\mathcal{S}$, and the set of actions is a finite set $\mathcal{A}$ with cardinality $A = |\mathcal{A}| < \infty$. The learner is provided with a dataset $D_n = ((S_1,R_1),\dotsc,(S_n,R_n))$, consisting of $n$ i.i.d. context-reward vector pairs, sampled from $\mathcal{P}^{\otimes n}$. Each sample satisfies $S_i \sim \mathcal{S}$ and $R_i \in [0,1]^{\mathcal{A}}$, where $\mathcal{P} \in \mathcal{M}_1(\mathcal{S} \times [0,1]^{\mathcal{A}})$ represents the joint distribution over contexts and reward vectors.

Define the expected reward function as
\[
r(s,a) = \mathbb{E}[R_1(a) \mid S_1 = s], \quad \forall s \in \mathcal{S}, \, a \in \mathcal{A}.
\]
A policy (classifier) $\pi$ maps contexts to probability distributions over actions, formally defined as $\pi: \mathcal{S} \to \mathcal{M}_1(\mathcal{A})$. We use $\pi(a \mid s)$ to denote the probability assigned by $\pi$ to action $a \in \mathcal{A}$ when the context is $s \in \mathcal{S}$. Let $\mathcal{P}_S \in \mathcal{M}_1(\mathcal{S})$ denote the marginal distribution over contexts. The expected return of following policy $\pi$ is
\[
v^\pi = \int \sum_{a \in \mathcal{A}} \pi(a \mid s) \, r(s,a) \, P_S(ds).
\]
An optimal policy $\pi^\star$ is one that maximizes the expected return across all policies:
\[
v^{\pi^\star} = v^\star = \max_{\pi} v^\pi.
\]
Define the \textit{suboptimality gap} (regret) of using policy $\pi$ instead of the optimal policy as
\[
\Subopt(\pi) = v^\star - v^\pi.
\]
The objective in classification is to learn a policy $\hat{\pi}$ that maximizes the expected return $v^{\hat{\pi}} = \hat{v}$, or equivalently, minimizes the suboptimality gap $\Subopt(\hat{\pi})$.

\subsection{Solution: Value-Based Methods for Classification}

In statistical learning theory \citep{vapnik2013nature}, the objective of classification is often to learn a policy $\hat{\pi}$ that minimizes the suboptimality gap $\Subopt(\hat\pi)$. Given a class of policies $\Pi \subset \mcM_1(\mcA)^{\mcS}$, a natural approach is the principle of empirical risk minimization on the dataset $D_n$ \citep{NIPS1991_ff4d5fbb}:
\[
\hat{\pi} = \argmax_{\pi \in \Pi} \sum_{i=1}^n \pi(\cdot \mid S_i)^\top R_i\,.
\]
However, directly optimizing this empirical risk is generally computationally intractable (NP hard), even for relatively simple policy classes \citep{ben2003difficulty,feldman2012agnostic}. This motivates the use of \textit{value-based methods} for classification, which leverage a class of candidate value functions to reformulate the problem as a regression task—often solvable via gradient descent.

Formally, given a dataset $D_n = ((S_1,R_1),\dotsc,(S_n,R_n))$, a realizable class (\cref{ass:realizability}) of candidate value functions $\mathcal{F} \subseteq [0,1]^{\mathcal{S} \times \mathcal{A}}$, and a loss function $\ell$, we solve the following regression problem:
\[
\hat{f} = \argmin_{f \in \mathcal{F}} \sum_{i=1}^n \sum_{a \in \mathcal{A}} \ell\bigl(f(S_i,a), R_i(a)\bigr).
\]
The learned function $\hat{f}$ is then used to define the greedy policy:
\[
\hat{\pi}(s) = \argmax_{a \in \mathcal{A}} \hat{f}(s,a).
\]
We refer to any method that minimizes a regression loss and then selects actions greedily with respect to the minimizer as a value-based method.

\subsection{The Choice of the Loss Function}
In the context of reinforcement learning, value-based methods commonly minimize the squared loss \citep{sutton2018reinforcement,szepesvari2022algorithms}:
\[
\ellsq(x,y) = (x - y)^2, \quad \text{where } x,y \in [0,1].
\]
This contrasts with traditional classification tasks, where the cross-entropy loss is mostly used. We remark that while ``cross-entropy loss'' is typically associated with the negative log-likelihood under Bernoulli or categorical models, the squared loss is the cross entropy loss that arises under a univariate Gaussian model. We introduce the terminology \textit{log-loss} and \textit{cat-loss} to denote the cross-entropy losses for Bernoulli and categorical distributions, respectively.

To develop intuition, we first focus on the log-loss, as it represents the simplest special case of the more general cat-loss.

We define the log-loss as
\[
\elllog(x,y) = y \log \frac{1}{x} + (1-y) \log \frac{1}{1-x}, \quad \text{where } x,y \in [0,1]\,,
\]
with the convention \(0 \log\infty = \lim_{u\to 0}\,u\log\frac{1}{u} = 0\).
The following proposition establishes that the minimizer of the log-loss is the population mean.

\begin{proposition}\label{prop:logloss-proper}
    Let $Y$ be a bounded random variable with $Y \in [0,1]$ and mean $\mathbb{E}[Y] = \mu$. Then, for any $x \in [0,1]$, the expected log-loss satisfies:
    \[
    \mathbb{E}[\elllog(x,Y)] \geq \mathbb{E}[\elllog(\mu,Y)].
    \]
\end{proposition}

\begin{proof}
    Recall the binary Kullback Leibler (KL) divergence
    \[
    \kl(p,q) = p \log \frac{p}{q} + (1-p) \log \frac{1-p}{1-q}.
    \]
    Then, observe that
    \begin{align*}
    \mathbb{E}[\ell_{\text{log}}(x,Y)] - \mathbb{E}[\ell_{\text{log}}(\mu,Y)]
    &= \mu \log \frac{1}{x} + (1-\mu) \log \frac{1}{1-x} - \mu \log \frac{1}{\mu} - (1-\mu) \log \frac{1}{1-\mu} \\
    &= \kl(\mu,x).
    \end{align*}
    Since the KL divergence is always nonnegative, with equality if and only if $x = \mu$, the result follows.
\end{proof}

This proposition implies that, given an infinite number of observations of a bounded random variable $Y \in [0,1]$, minimizing the log-loss \emph{recovers the mean} of $Y$. In later sections, we will extend this insight to the cat-loss---the canonical loss for multi-class classification.


\begin{remark}
    Although our goal is to control the mean absolute error, minimizing the absolute loss
    \(\ell_{1}(x,y) \!=\! |x-y|\) is ill-suited to this purpose: its population risk minimizer is the
    \emph{median}, not the mean. Greedy decision making with respect to the median is typically
    suboptimal. Hence, it is preferable to employ surrogate loss functions whose
    population minimizers coincide with the mean, while still affording tight control of the
    mean absolute error.
    \end{remark}

\section{Mean Absolute Error as a More Natural Objective for Decision Making}

We now present the central insight of this paper: \emph{the mean absolute error (MAE), rather than the mean squared error (MSE), is a more suitable regression objective for decision-making problems} such as reinforcement learning and classification. Formally, let $\hat{\pi}$ denote the greedy policy with respect to some learned function $\hat{f} \in [0,1]^{\mathcal{S} \times \mathcal{A}}$. For a score function $f \in [0,1]^{\mathcal{S} \times \mathcal{A}}$ and a deterministic policy $\pi: \mathcal{S} \to \mathcal{A}$, define
\[
f(s,\pi) = f\bigl(s, \pi(s)\bigr).
\]
We now bound the suboptimality gap $\Subopt(\hat{\pi})$. Analyses of value-based methods for classification typically reduce a classification objective (i.e., $\Subopt(\hat{\pi})$) to a regression objective (i.e., $\mathsf{VE}_p(\hat{f})$); see, for instance, \cite{antos2007fitted,chen2019information,ayoub2024switching} for batch RL and \cite{ayoub2020model,jin2021bellman,jin2023provably} for online RL. In doing so, one obtains both a mean absolute error (MAE) term and a root mean squared error (rMSE) term. Indeed, observe that
\begin{align}
    \Subopt(\hat{\pi})
    &= \int r(s,\pi^\star) - r(s,\hat{\pi})\,P_S(ds) \nonumber\\[6pt]
    &= \int r(s,\pi^\star) - \hat{f}(s,\hat{\pi}) + \hat{f}(s,\hat{\pi}) - r(s,\hat{\pi})\,P_S(ds) \nonumber\\
    &\leq \int r(s,\pi^\star) - \hat{f}(s,\pi^\star) + \hat{f}(s,\hat{\pi}) - r(s,\hat{\pi})\,P_S(ds) 
    \label{eqn:class}\\
    &\leq \int \bigl|r(s,\pi^\star) - \hat{f}(s,\pi^\star)\bigr|\,P_S(ds) + \int \bigl|\hat{f}(s,\hat{\pi}) - r(s,\hat{\pi})\bigr|\,P_S(ds)\tag{MAE}\\
    &\leq \sqrt{\int \bigl(r(s,\pi^\star) - \hat{f}(s,\pi^\star)\bigr)^2\,P_S(ds)} + \sqrt{\int \bigl(r(s,\hat{\pi}) - \hat{f}(s,\hat{\pi})\bigr)^2\,P_S(ds)}\,.\tag{rMSE}
\end{align}
The first inequality in \eqref{eqn:class} uses the fact that $\hat{\pi}$ is greedy with respect to $\hat{f}$. The last inequality applies Jensen's inequality. Notice that the key reduction from a policy-space comparison ($\hat{\pi}$ vs.\ $\pi^\star$) to a function-space comparison ($\hat{f}$ vs.\ $r$) occurs in \eqref{eqn:class}.

\smallskip
\noindent
\textbf{Motivation for MAE vs.\ rMSE.} As shown, both MAE and rMSE naturally arise in bounding $\Subopt(\hat{\pi})$. However that MAE is a \emph{tighter} approximation to $\Subopt(\hat{\pi})$ than rMSE, since 
\[
\int |f - g| \leq  \sqrt{\int (f - g)^2}\,.
\]
In practical settings, there are cases where $\int |r - \hat{f}|$ is significantly smaller than $\sqrt{\int (r-\hat{f})^2}$, implying that algorithms designed to control rMSE (such as squared-loss minimization) can incur larger suboptimality than algorithms targeted toward controlling MAE (such as log-loss minimization). Since the ultimate goal in decision making is to select good actions, it is natural to adopt a regression metric (and loss) that is more closely aligned with suboptimality. We now present a set of results that confirm our intuition.  

\subsection{Positive Results}
We highlight that minimizing the log-loss (i.e., $\elllog$) yields bounds that scale with $(1-v^\star)$, which can be small in problems where the optimal policy achieves a reliable goal or accumulates near-maximal return. The following lemma adapts the result of \cite{foster2021efficient} to the rewards-based setting.

\begin{lemma}\label{lem:log-loss-bound}
    Assume $r \in \mcF$ and define
    \[
    \hat{f}_{\log} \in \argmin_{f \in \mcF} \sum_{i=1}^n \sum_{a \in \mathcal{A}} \elllog\bigl(f(S_i,a), \,R_i(a)\bigr)\,.
    \]
    Let $\hat{\pi}_{\log}$ be the greedy policy w.r.t.\ $\hat{f}_{\log}$. Then with probability $1-\delta$,
    \[
    \Subopt(\hat{\pi}_{\log})
    \leq
    16\,\sqrt{\frac{2\,(1 - v^\star)\,A \,\log\bigl(|\mcF|/\delta\bigr)}{n}} 
    + 
    \frac{136\,A\,\log\bigl(|\mcF|/\delta\bigr)}{n}\,,
    \]
    and
    \begin{align*}
        & \int \bigl|r\bigl(s,\pi^\star\bigr) - \hat{f}\bigl(s,\pi^\star\bigr)\bigr|P_S(ds)
        +
        \int \bigl|\hat{f}\bigl(s,\hat{\pi}\bigr) - r\bigl(s,\hat{\pi}\bigr)\bigr|P_S(ds)\\[4pt]
        &\qquad\leq
        16\,\sqrt{\frac{2\,(1 - v^\star)\,A\,\log\bigl(|\mcF|/\delta\bigr)}{n}}
        +
        \frac{136\,A\,\log\bigl(|\mcF|/\delta\bigr)}{n}\,.
    \end{align*}
\end{lemma}
The proof of \cref{lem:log-loss-bound} can be found in \cref{appendix:proofs}.
This result first appeared for cost-sensitive classification \citep{foster2021efficient} and was extended to batch RL (with costs) in \cite{ayoub2024switching}. By adapting their proofs, one obtains the same bound for batch RL with rewards. Observe that their analysis also proceeds by reducing the classification (or RL) objective to bounding the MAE between $\hat{f}_{\log}$ and $r$. We now show that, under certain conditions ($v^\star \approx 1$), one can achieve small MAE while the rMSE remains large.

\subsection{Negative Results}
In this section, we argue that rMSE is not the most ``natural objective function'' (Chapter 9.2 of \citep{sutton2018reinforcement}) for RL. We make two claims. The first claim (\cref{prop:rmse-lb}) is that there are problems where the rMSE of $\hat{f}_{\log}$ is at least $\tfrac{1}{\sqrt{n}}$, while its MAE is as small as $\tfrac{1}{n}$. The second claim (\cref{lem:sq-mae-lb}) is that if one directly minimizes the squared loss $\ellsq$, then there exist problems for which the resulting estimator suffers a large MAE (and hence rMSE) of order $1/\sqrt{n}$, even when $v^\star \approx 1$.

\begin{proposition}\label{prop:rmse-lb}
    Let $\mathcal{X} = \{x, x'\}$. For every $n \ge 1$, there exists a realizable (\cref{ass:realizability}) function class  $\mathcal{F}: \mathcal{X}\to[0,1]$ with $|\mathcal{F}| = 2$ and a data distribution $\mathcal{P}$ such that $1-v^\star = 1-\int f^\star\,P_X(dx) < \tfrac{1}{n}$. However, with probability at least $1/(2e)$
    \[
    \sqrt{\int \bigl(f^\star(x) - \hat{f}_{\log}(x)\bigr)^2 P_X(dx)} 
    =
    \frac{1}{2\sqrt{n}}\,.
    \]
\end{proposition}

In the construction of \cref{prop:rmse-lb}, \cref{lem:log-loss-bound} implies that 
\[
\int \bigl|f^\star(x) - \hat{f}_{\log}(x)\bigr|P_X(dx)
\in 
O\left(\frac{1}{n}\right),
\]
yet the rMSE remains at order $\tfrac{1}{\sqrt{n}}$. Intuitively, context $x'$ appears with low probability ($1/n$), so the log-loss estimator $\hat{f}_{\log}$ might be quite noisy on $x'$ but accurate on the high-probability context $x$. Since rMSE weights the rare event $x'$ more heavily than MAE does, it overestimates the error relevant for \emph{decision making} (which focuses on probable states). The formal details can be found in \cref{appedix:lower-bounds}.

Finally, we show that this ``over-weighting of rare events'' carries over to \emph{minimizing the squared loss} directly, as commonly done by value-based RL algorithms with function approximation \citep{sutton2018reinforcement}. 

\begin{lemma}\label{lem:sq-mae-lb}
    Let $\mathcal{X} = \{x, x'\}$. For any $n \ge 1$, there exists a realizable (\cref{ass:realizability}) function class $\mathcal{F}: \mathcal{X}\to[0,1]$ with $|\mathcal{F}|=2$ and a data distribution $\mathcal{P}$ such that $1-v^\star = 1-\int f^\star\,P_X(dx) < \tfrac{1}{n}$. However, with probability at least $1/(2e)$,
    \[
    \int \bigl|\fsq(x) - f^\star(x)\bigr|P_X(dx) 
    \ge 
    \frac{1}{3\sqrt{n}}\,,
    \]
    where
    \[
    \fsq
    =
    \argmin_{f \in \mathcal{F}} 
    \sum_{i=1}^n \ellsq\bigl(f(X_i),Y_i\bigr).
    \]
\end{lemma}
The proof of \cref{lem:sq-mae-lb} can be found in \cref{appedix:lower-bounds}.
In cost-sensitive classification, \cite{foster2021efficient} showed an analogous phenomenon: the greedy policy with respect to the squared-loss minimizer can incur $\tfrac{1}{\sqrt{n}}$-sized suboptimality even though $v^\star \ge 1 - \tfrac{1}{n}$. Our lemma complements their result by demonstrating that the squared-loss minimizer itself fails to achieve a $O(1/n)$ decay in its mean absolute error, whereas the log-loss minimizer can achieve $O(1/n)$ decay in similar settings. Collectively, \cref{lem:log-loss-bound}, \cref{prop:rmse-lb}, and \cref{lem:sq-mae-lb} highlight that the MAE (and hence the log-loss objective) is more tightly coupled to $\Subopt(\hat{\pi})$ than rMSE is. Moreover, algorithms specifically designed to control the MAE---such as minimizing $\elllog$ when $v^\star \approx 1$---can adapt more effectively to problem structure than those designed around controlling rMSE (via $\ellsq$).

\section{Reparameterizing the Categorical Cross Entropy Loss}
We have seen that the log-loss can outperform the squared loss in decision-making tasks, particularly when \(v^\star\) is close to 1. A natural next step is to seek a multi-category version of log-loss that retains its ability to learn the mean with sufficient data. This can be accomplished by \emph{Reparameterizing} the categorical cross-entropy loss so that it serves as both a ``classification'' and a ``regression'' loss.

Recall that the canonical categorical cross-entropy (cat-loss) can be written as the negative log-likelihood of an exponential family \citep{brown1987fundamentals}. Let \(y \in [0,1]\) be a \emph{scalar} and \(\theta \in \mathbb{R}^K\). In its canonical form, the cat-loss can be (naively) used for value learning in reinforcement learning,
\[
\ell(\theta,y) 
=
\log\Bigl(\sum_{i=1}^{K} \exp(\theta_i)\Bigr) 
- 
T(y)^\top \theta,
\]
where \(T(y)\) ``bins'' \(y\) into one of \(K\) discrete categories. Concretely,
\[
T(y) 
=
\bigl[\,
\II\{\,0 \le y \le \nu_1\},
\II\{\nu_1 < y \le \nu_2\},
\dots,
\II\{\nu_{K-1} < y \le 1\}
\bigr]\,,
\]
where $0 < \nu_1 < \nu_2 < \dotsc < \nu_{K-1} < 1$. 
Unfortunately, this form of the cat-loss introduces an irreducible \emph{projection bias} for regression tasks since the exact location of \(y\) within the bin is lost. This bias prevents accurate value function estimation.

\paragraph{Reparameterized Cat-Loss.}
To remove this bias while retaining the multi-category structure, we reparameterize the loss to incorporate \(y\) directly into the sufficient statistic. Define
\[
\ellcat(\theta,y) 
=
\underbrace{\log \Bigl(1+\sum_{i=1}^{K-1} \exp\bigl(\nu_i\,\theta_i \bigr) + \exp(\theta_{K})\Bigr) }_{\displaystyle = A(\theta)}
-
y \,T(y)^\top \theta.
\]
Here, \(A(\theta)\) is the log-partition function, and the sufficient statistic is \(y\,T(y)\). Thus, \(\ellcat\) remains the negative log-likelihood of an exponential family, but one that \emph{does not} lose all fine-grained information about \(y\) through binning. The following proposition shows that \(\ellcat\) preserves the mean of any bounded scalar random variable.

\begin{proposition}\label{prop:cat-loss-mean}
    Let \(Y\) be a bounded random variable taking values in \([0,1]\) with mean \(\mu = \mathbb{E}[Y]\), and let \(P\) be its distribution. Then
    \[
    \theta_\star 
    \in
    \argmin_{\theta \in \mathbb{R}^K} 
    \int \ellcat(\theta, y)\,P(dy) 
    \iff
    \bigl(\nabla A(\theta_\star)\bigr)^\top \mathbf{1} 
    = 
    \mu.
    \]
\end{proposition}
\begin{proof}
    Since \(\ellcat\) is convex in \(\theta\), first-order optimality conditions imply
    \begin{equation*}
    \theta_\star \in \argmin_{\theta \in \mathbb{R}^K} \int \ellcat(\theta,y)\,P(dy)
    \iff
    \nabla \Bigl(\int \ellcat(\theta_\star, y)\,P(dy)\Bigr) = 0\,.   
    \end{equation*}
    Differentiating under the integral, we obtain
    \begin{equation*}
    \nabla A(\theta_\star) 
    =
    \int y\,T(y)\,P(dy).
    \end{equation*}
    Next, observe that \(T(y)\) is a one-hot bin indicator, so
    \begin{equation*}
    \Bigl(\int y\,T(y)\,P(dy)\Bigr)^\top \mathbf{1} 
    =
    \int \bigl(y\,T(y)^\top \mathbf{1}\bigr)\,P(dy)=\int y\,P(dy)
    =
    \mu\,.
    \end{equation*}
    It follows that 
    \(\bigl(\nabla A(\theta_\star)\bigr)^\top \mathbf{1} = \mu\), establishing the claim.
\end{proof}

Hence, minimizing \(\ellcat\) recovers the mean of any bounded scalar random variable \(Y\) while the structure of \(T(y)\) still allows multi-category classification. In essence, by choosing the sufficient statistic to be \(y\,T(y)\), we unify classification and regression without the loss of granularity inherent in a purely bin-based approach.

Since the cat-loss is a generalization of the log-loss, a first-order bound, similar to that of \cref{lem:log-loss-bound}, can be shown for the cat-loss. Thus the cat-loss also does well when the optimal return $v^\star \approx 1$, albeit with an additional factor of $K$ scaling the bound due to concentration arguments \citep{zhang2006varepsilon,grunwald2020fast}.

\paragraph{Extensions.}
This mean-preserving strategy can be generalized to other distribution-based losses derived from exponential families. For instance, \emph{HL Gauss} \citep{imani2018improving}, which has been shown to have very low test MAE \citep{imani2024investigating}, can be reparameterized similarly to retain its distributional modeling benefits for classification, while accurately recovering the mean of continuous targets. In a broader sense, \emph{any} exponential-family log-likelihood can be adapted for dual ``classification''-regression usage by carefully selecting sufficient statistics that embed the raw target \(y\), thus allowing us to harness the benefits of classification\footnote{Cross-entropy losses} losses (i.e., categorical cross entropy) for reinforcement learning and continue regressing \citep{farebrother2024stop}.

\section{Numerical Experiments}

We evaluate fitted Q-iteration \citep{ernst2005tree,szepesvari2022algorithms} trained with squared loss (\(\fqisq\)), log-loss (\(\fqilog\)), and cat-loss (\(\fqicat\)) on the inverted pendulum environment \citep{lagoudakis2003least,riedmiller2005neural}, where the goal is to keep an inverted pendulum balanced by applying the correct forces. The state space is two-dimensional (angle and angular velocity), and there are three actions (left, right, do nothing). The environment dynamics follow \cite{lagoudakis2003least}, with two modifications: (i) when the pendulum falls below horizontal, the state terminates, and (ii) the angular momentum is clipped to [-5,5] to facilitate the use of Fourier features \citep{konidaris2011value} of orders 2 and 3. The agent receives a reward of 0 for staying upright and -1 for falling, with a discount factor of $\gamma = 0.99$. All datasets are collected by a policy that selects actions uniformly at random until failure (which typically occurs after 6 steps). We then evaluate whether the learned policies keep the inverted pendulum above horizontal after 3000 steps and report the policy's \textit{failure rate}.

\begin{figure}[tbp]\label{fig:main}
    \centering
    \begin{subfigure}{0.32\textwidth}
        \centering
        \includegraphics[width=\linewidth]{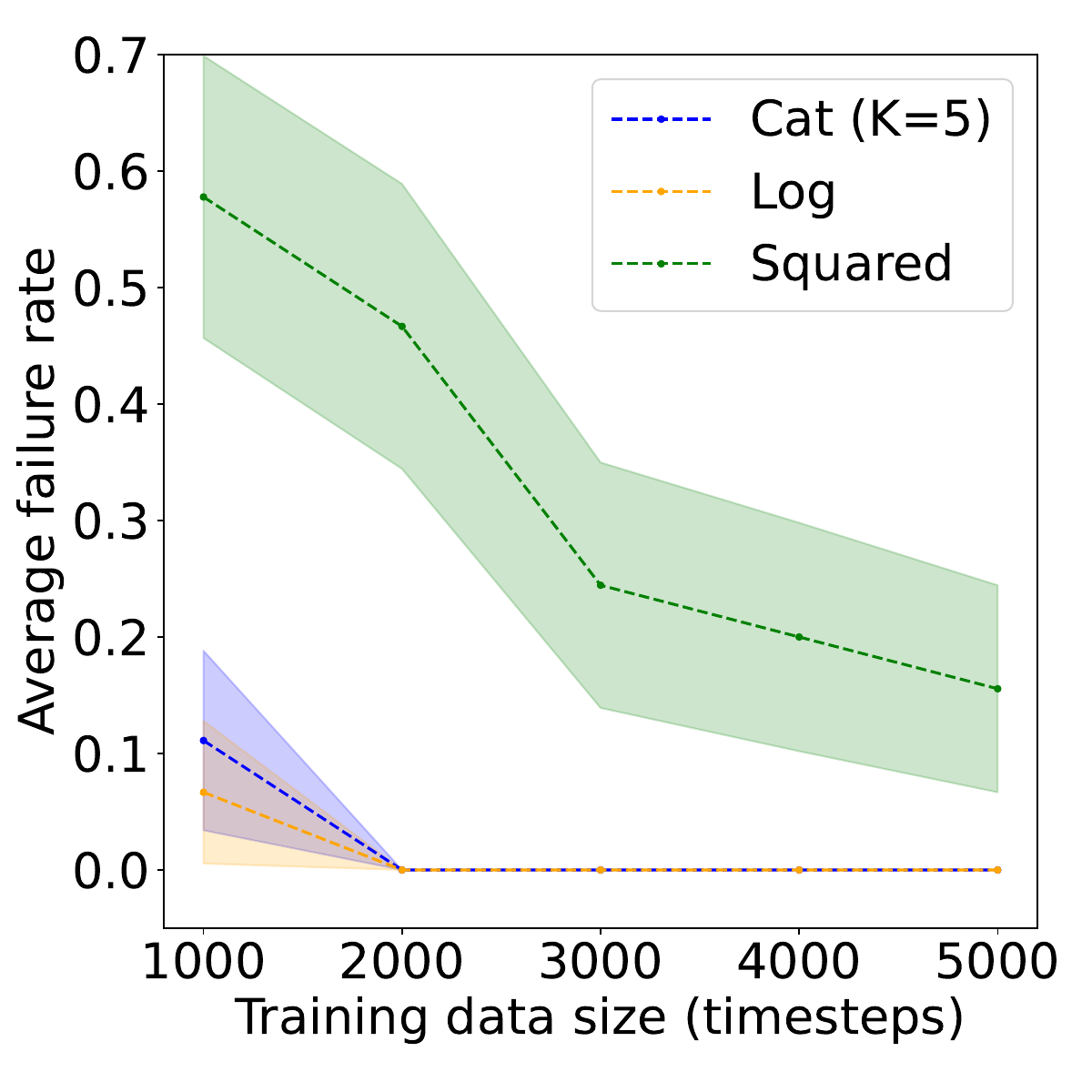}
    \end{subfigure}
    \begin{subfigure}{0.32\textwidth}
        \centering
        \includegraphics[width=\linewidth]{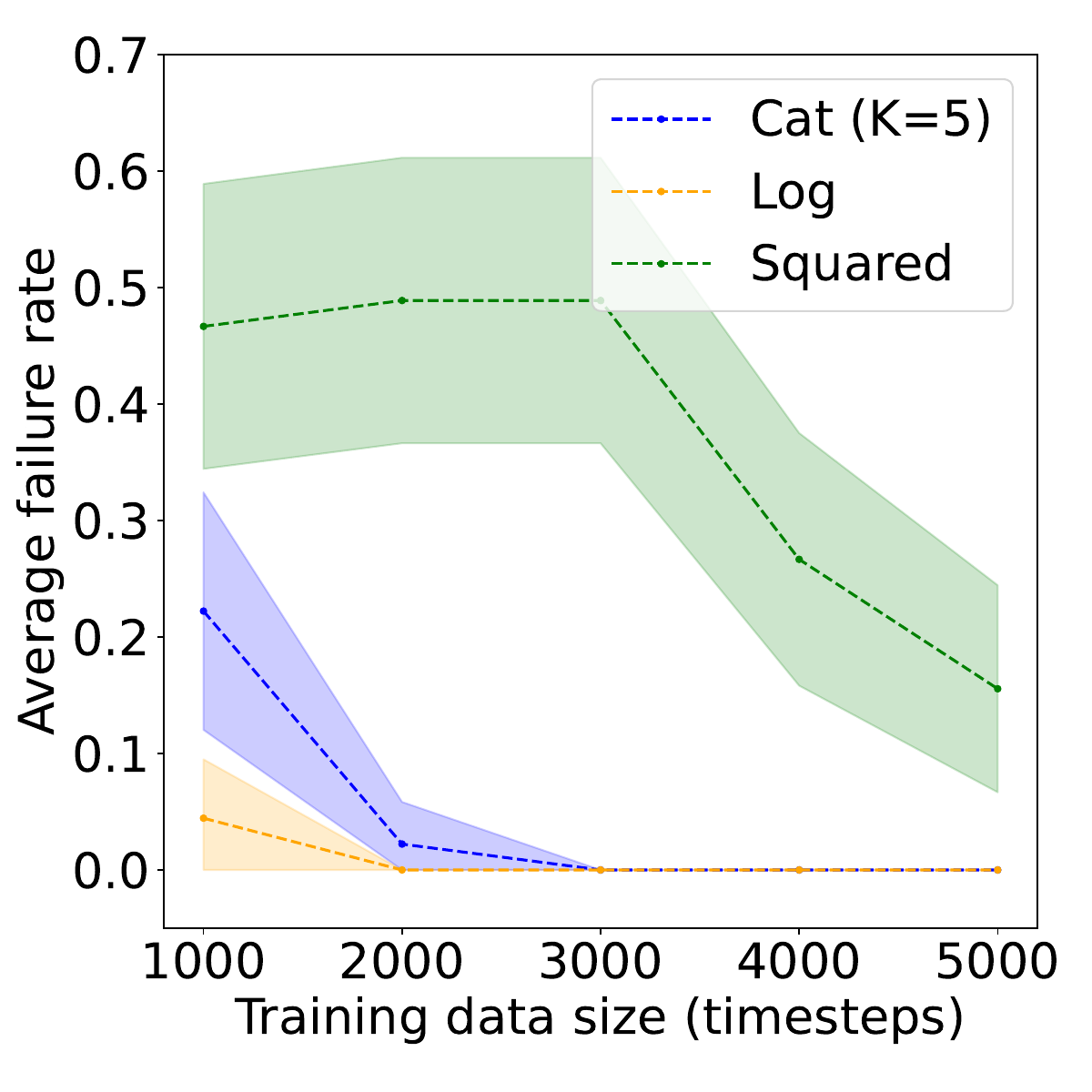}
    \end{subfigure}
    \begin{subfigure}{0.32\textwidth}
        \centering
        \includegraphics[width=\linewidth]{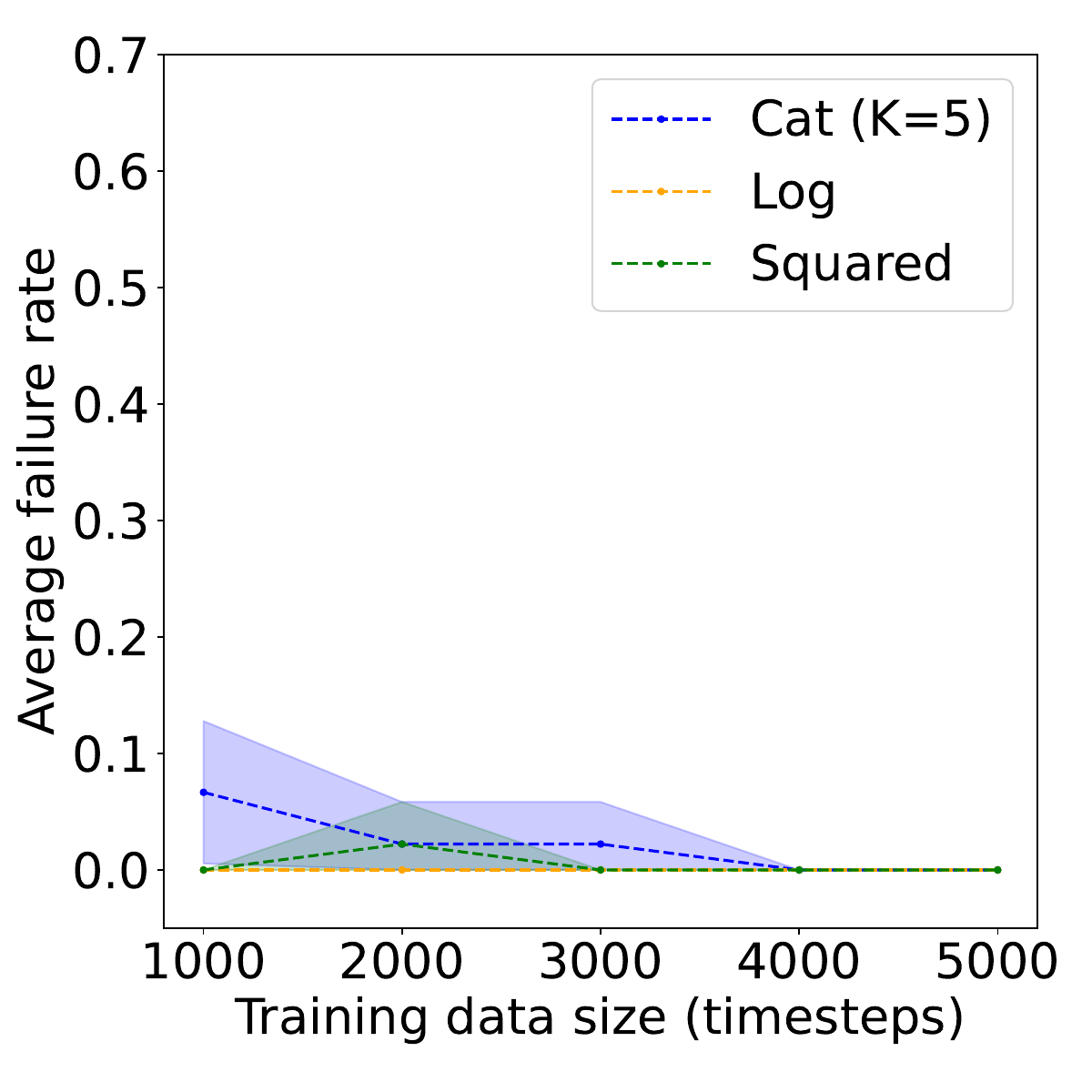}
    \end{subfigure}
    \caption{Failure rates for inverted pendulum as a function of the size of the batch dataset. Results are averaged over $45$ independently collected datasets, and fitted $Q$-iteration was run for $50$ iterations. We report $90\%$ confidence intervals via the shaded regions. The \textsc{Left} and \textsc{Middle} figures use Fourier features of order $2$, the \textsc{Right} figure uses Fourier features of order $3$. The \textsc{Left} figure uses $5$ uniformly spaced points as the support for the \fqicat, while the \textsc{Middle} and \textsc{Right} figures use $5$ non-uniformly spaced points as the support.}
    \label{fig:three_plots1}
\end{figure}

For minimizing \(\fqisq\), we use its closed-form solution; for minimizing \(\fqilog\), we apply Newton’s method \citep{sun2019generalized}; and for minimizing \(\fqicat\), we use a limited-memory BFGS method \citep{pmlr-v5-schmidt09a}. All three methods are guaranteed to converge to their respective optima superlinearly.

In this environment, there is a policy that achieves the maximum possible return. As shown in \cref{fig:three_plots1}, both \(\fqilog\) and \(\fqicat\) outperform \(\fqisq\) with order-2 Fourier features, while all three perform similarly with order-3 features, though \(\fqilog\) learns fastest. While \(\fqilog\) and \(\fqisq\) were relatively straightforward to implement, \(\fqicat\) proved more sensitive to both the choice of $\ell_2$ regularization and the spacing and number of bins.

\section{Conclusion}
In this paper, we examined how the choice of regression objective can influence the design of value-based methods in reinforcement learning. We showed that losses aligning with mean absolute error, such as log-loss and a reparameterized categorical loss, can yield stronger theoretical guarantees and better empirical outcomes than squared loss, especially when the optimal policy is near maximum return. We also presented negative examples illustrating that purely MSE-based approaches can learn slowly in such scenarios. Finally, our experiments in a linear batch reinforcement learning setting reinforce these conclusions. 

\subsubsection*{Acknowledgments}
\label{sec:ack}
This work was supported by Amii, the Canada CIFAR
AI Chairs program, and the Natural Sciences and Engineering Research Council of Canada (NSERC).

\appendix


\bibliography{main}
\bibliographystyle{rlj}

\beginSupplementaryMaterials

\section{Additional Notation}
In this section, we introduce additional notation we will find useful in stating our theoretical results. For a distribution $P$ over the reals, let $\EE(P)$ denote the mean of $P$. Furthermore for $p,q \in [0,1]$ define the pointwise triangular deviation between $p,q$ as 
$$
\Delta(p,q) = \frac{(p-q)^2}{p+q}
$$
and the binary Hellinger distance of $p$ and $q$ as
$$
\hell(p,q) = \frac{1}{2}(\sqrt{p}-\sqrt{q})^2 + \frac{1}{2}(\sqrt{1-p}-\sqrt{1-q})^2\,.
$$
Furthermore for any $x \in [0,1]$, define 
$$
L(x) = 1-x\,.
$$ 
\section{Proof of \cref{lem:log-loss-bound}}\label{appendix:proofs}
We begin by bounding the suboptimality gap of the value-based method that minimizes the log-loss \(\elllog\).

\begin{proof}[Proof of \cref{lem:log-loss-bound}]
    By \cref{lem:dylan}, we have
    \begin{align*}
        \Subopt(\hat{\pi})
        &= v^\star - \hat{v}
        =
        L(\hat{v}) - L(v^\star) \\
        &\leq 
        8 \sqrt{\,L(v^\star)\,\mathbb{E}\!\Bigl[\sum_{a\in\mcA}\Delta\!\bigl(L(f^\star(S,a)),\,L(\hat{f}(S,a))\bigr)\Bigr]} 
        + 
        17\, \mathbb{E}\!\Bigl[\sum_{a \in \mcA}\Delta\!\bigl(L(f^\star(S,a)),\,L(\hat{f}(S,a))\bigr)\Bigr].
    \end{align*}
    Next, by \cref{lem:tri-to-hell},
    \[
    \Delta\!\bigl(L(f^\star(S,a)),\,L(\hat{f}(S,a))\bigr) 
    \leq
    4\,\hell\!\bigl(f^\star(S,a),\,\hat{f}(S,a)\bigr).
    \]
    Applying \cref{thm:concentration}, we get
    \[
    \mathbb{E}\,\Bigl[\sum_{a \in \mcA}
    \Delta\!\bigl(L(f^\star(S,a)),\,L(\hat{f}(S,a))\bigr)\Bigr] 
    \leq 
    \frac{2A\,\log\bigl(|\mcF|/\delta\bigr)}{n}.
    \]
    The second part of the lemma follows from the argument in the proof of Lemma 1 in \cite{foster2021efficient}, where bounding the mean absolute error appears as an intermediate step.
\end{proof}

\section{Proof of the Negative Results}\label{appedix:lower-bounds}

In this section, we show that minimizing the empirical squared loss does not always achieve a \(1/n\)-rate for the mean absolute error (MAE) when \(\int f^\star(x)\,P_X(dx)\approx 1\). By contrast, \cref{lem:log-loss-bound} implies that minimizing the empirical log-loss \emph{does} achieve such a rate under the same conditions.

\medskip

\noindent
\textbf{Setup.}
Given a dataset \(\{(X_i,Y_i)\}_{i=1}^n \sim \mcP^{\otimes n}\) with \(Y_i \in [0,1]\) and \(X_i \in \mcX\), define the empirical squared loss of a function \(f: \mcX \rightarrow [0,1]\) by
\[
\hat{L}(f) 
= 
\sum_{i=1}^n \bigl(f(X_i) - Y_i\bigr)^2.
\]
Let \(P_X\) denote the distribution of the contexts. Given a fixed function class \(\mcF \subseteq [0,1]^\mcX\), define the empirical risk minimizer (ERM) for squared loss as
\[
\fsq 
=
\argmin_{f \in \mcF} \hat{L}(f).
\]
We now construct a problem instance where \(\fsq\) does not achieve \(1/n\)-rate convergence for MAE, even though \(f^\star \approx 1\). Recall that by \cref{lem:log-loss-bound}, the empirical log-loss minimizer \emph{can} achieve a \(1/n\)-rate under similar conditions.

\paragraph{Construction.}
Let \(\mcX = \{x,x'\}\) and set \(P_X(x) = 1 - 1/n\) and \(P_X(x') = 1/n\). The labels \(Y\) have the following conditional distributions:
\begin{enumerate}
    \item \(Y \mid x = 1 - \tfrac{1}{2n}\) almost surely, so \(\mathbb{E}[Y \mid X = x] = f^\star(x) = 1 - \tfrac{1}{2n}\).
    \item \(Y \mid x' \sim \mathrm{Bernoulli}(1/2)\), so \(\mathbb{E}[Y \mid X = x'] = f^\star(x') = 1/2\).
\end{enumerate}
We take the function class \(\mcF = \{f^\star,\psi\}\), where
\[
\psi(x) = 1 - \frac{1}{2n} - \frac{1}{3\sqrt{n}},
\quad \text{and} \quad
\psi(x') = 0.
\]
This class satisfies the realizability assumption (i.e., \(f^\star \in \mcF\)).

\begin{proof}[Proof of \cref{lem:sq-mae-lb}]
    Suppose \(\hat{L}(\psi) \leq \hat{L}(f^\star)\). Then \(\fsq = \psi\) and thus
    \[
    \int \bigl|\fsq(x) - f^\star(x)\bigr|P_X(dx) 
    \ge 
    \Bigl|\,\Bigl(1-\frac{1}{2n}-\frac{1}{3\sqrt{n}}\Bigr) - \Bigl(1-\frac{1}{2n}\Bigr)\Bigr|
    = 
    \frac{1}{3\sqrt{n}}.
    \]
    It remains to show that \(\hat{L}(\psi) \leq \hat{L}(f^\star)\) holds with constant probability. Let \(N_2\) be the number of times \(X_i = x'\) in the dataset. Then
    \[
    \mathbb{P}(N_2 = 1) 
    =
    \sum_{i=1}^n \frac{1}{n}\Bigl(1-\frac{1}{n}\Bigr)^{n-1} 
    =
    \frac{1}{\,1 - \tfrac{1}{n}\,} \Bigl(1-\tfrac{1}{n}\Bigr)^{n-1}
    \ge 
    \frac{1}{e},
    \]
    for all \(n \ge 1\). Conditioning on the event \(N_2=1\), we have exactly one observation of \(x'\). Since \(Y=0\) in that observation with probability \(1/2\), it follows that with probability at least \(1/(2e)\) we observe a single \((x',0)\) point in the dataset. On this event,
    \[
    \hat{L}(f^\star) - \hat{L}(\psi) 
    = 
    \frac{1}{4} - (n-1)\Bigl(\frac{1}{3\sqrt{n}}\Bigr)^2 
    \ge
    \frac{1}{4} - \frac{1}{9}
    > 
    0.
    \]
    Hence \(\fsq = \psi\) with probability at least \(1/(2e)\) for all \(n\ge1\).
\end{proof}

\paragraph{Proof of \cref{prop:rmse-lb}.}
We use a similar construction to prove that the root mean squared error (rMSE) of \(\hat{f}_{\log}\) can remain at \(\Omega(1/\sqrt{n})\) even when \(\int f^\star(x)\,P_X(dx) \approx 1\). We slightly modify the function class to
\[
\mcF_{\log} 
=
\{\,f^\star,\,\phi\},
\quad
\text{where }
\phi(x) = f^\star(x),
\phi(x') = 0.
\]
Define
\[
\hat{L}_{\log}(f)
=
\sum_{i=1}^n \elllog\bigl(f(X_i),Y_i\bigr),
\quad
\hat{f}_{\log}
=
\argmin_{f \in \mcF_{\log}} \hat{L}_{\log}(f).
\]
If \(\hat{L}_{\log}(\phi) \le \hat{L}_{\log}(f^\star)\), then \(\hat{f}_{\log} = \phi\), and
\[
\sqrt{
\int \bigl(\hat{f}_{\log}(x) - f^\star(x)\bigr)^2\,P_X(dx)
}
=
\frac{1}{2\sqrt{n}}.
\]
Let \(N_2\) be the number of times \(X_i = x'\). As in the proof of \cref{lem:sq-mae-lb},
\[
\mathbb{P}(N_2 = 1)
\ge
\frac{1}{e}.
\]
Conditioning on \(N_2 = 1\), we have \(Y=0\) at \(x'\) with probability \(1/2\), which occurs with probability \(1/(2e)\). On this event,
\[
\hat{L}_{\log}(f^\star)
-
\hat{L}_{\log}(\phi)
=
\log \frac{1}{1 - 0.5}
-
\log \frac{1}{1}
=
\log(2)
>
0,
\]
so \(\hat{f}_{\log} = \phi\). Hence with probability at least \(1/(2e)\), the rMSE between \(\hat{f}_{\log}\) and \(f^\star\) remains \(\tfrac{1}{2\sqrt{n}}\).

\medskip

Thus, we have shown that there are problems where \(v^\star \approx 1\) but the rMSE of \(\hat{f}_{\log}\) decays no faster than \(1/\sqrt{n}\). 
\section{Technical Results}
\begin{lemma}[Lemma 1 of \cite{foster2021efficient}]\label{lem:dylan}
For any function $f : \mcS \times \mcA \rightarrow [0,1]$ and policy $\pi$ that is greedy with respect to $f$,
\begin{align*}
    \MoveEqLeft L(v^\pi) - L(v^\star)\\
    &\leq 8 \sqrt{L(v^\star)\int\,\sum_{a\in\mcA}\Delta\big(L^\star (s,a),\hat L(s,a)\big)\,P_S(ds)} + 17 \int\,\sum_{a \in \mcA}\Delta\big(L^\star (s,a),\hat L(s,a)\big)\,P_S(ds)
\end{align*}
where $L^\star(s,a) = L(f^\star(s,a))$ and $\hat L(s,a) = L(\hat f(s,a))$.
\end{lemma}
\begin{lemma}[Lemma A.1 of \cite{ayoub2024switching}]\label{lem:tri-to-hell}
    For all $p,q \in [0,1]$, we have
    \begin{equation*}
        \frac{1}{4}\Delta(p,q) \leq \frac{1}{2}(\sqrt{p}-\sqrt{q})^2 \leq \hell(p,q)\,.
    \end{equation*}
\end{lemma}

\subsection{Concentration for the Log-Loss Estimator}
Fix a context set \(\mcX\). Let \(\{(X_i,Y_i)\}_{i=1}^n\) be i.i.d.\ samples from a distribution \(\nu \in \mathcal{M}_1(\mcX\times[0,1])\). Define the regression function 
\[
f^\star(x) = \EE\bigl[Y_1 \,\big|X_1 = x\bigr].
\]
Suppose we have a finite class of candidate functions \(\mcF \subseteq [0,1]^{\mcX}\). Recall that the log-loss estimator is given by
\[
\hat{f}_{\log} 
=
\argmin_{f\in \mcF}
\sum_{i=1}^n
\elllog\bigl(f(X_i),\,Y_i\bigr),
\]
where for \(x,y\in [0,1]\),
\[
\elllog(x,y) 
= 
y \,\log \!\frac{1}{x}
+
(1-y)\,\log \!\frac{1}{1-x},
\]
with the convention \(0 \log\infty = \lim_{u\to 0}\,u\log\frac{1}{u} = 0\).

\citet{foster2021efficient} establish the following concentration result for \(\hat{f}_{\log}\). We restate it here for completeness.

\begin{theorem}\label{thm:concentration}
    Suppose \(f^\star \in \mcF\). Let \(D_n = \{(X_i,Y_i)\}_{i=1}^n \sim \nu^{\otimes n}\). Then for any \(\delta\in(0,1)\), with probability at least \(1-\delta\),
    \[
    \int \hell\!\bigl(\hat{f}_{\log}(x),\,f^\star(x)\bigr)\,\nu_X(dx)
    \le
    \frac{2\,\log\bigl(|\mcF|/\delta\bigr)}{n},
    \]
    where \(\nu_X\) is the marginal distribution of \(X_1,\dots,X_n\).
\end{theorem}

\begin{proof}
    The result follows directly from the last equation on page 24 of the arXiv version of \citet{foster2021efficient}, taking \(A=1\). \qedhere
\end{proof}

\end{document}